\documentclass[10pt]{IEEEtran}
\usepackage[table]{xcolor}
\usepackage{pgfplotstable}
\usepackage{fullpage,graphicx,psfrag,amsmath,amsfonts,verbatim,amssymb,amsthm}
\usepackage{braket}
\usepackage[small,bf]{caption}
\usepackage{subcaption}
\usepackage{multicol}
\usepackage{cite}
\usepackage{subcaption}
\usepackage{algorithm}% http://ctan.org/pkg/algorithms
\usepackage[noend]{algpseudocode}% http://ctan.org/pkg/algorithmicx
\usepackage{cite}
\usepackage{placeins}
\usepackage[utf8]{inputenc}
\usepackage[english]{babel}
\usepackage{lipsum}
\usepackage{textcomp}
\usepackage{placeins}
\usepackage[T1]{fontenc}
\usepackage{dsfont}
\usepackage{tabulary}

\newcommand{\argmin}{\mathop{\rm argmin}}

\newtheorem{theorem}{Theorem}
\newtheorem{lemma}[theorem]{Lemma}
\newcolumntype{K}[1]{>{\centering\arraybackslash}m{#1}}

\title{Shape-Based Approach to Household Load Curve Clustering and Prediction}

\author{\IEEEauthorblockN{Thanchanok Teeraratkul\IEEEauthorrefmark{1},
Daniel O'Neill\IEEEauthorrefmark{2} and Sanjay Lall\IEEEauthorrefmark{3} }
\IEEEauthorblockA{Electrical Engineering \\
Stanford University\\
Stanford, USA\\
e-mail: \IEEEauthorrefmark{1}tteerara@stanford.edu,
\IEEEauthorrefmark{2}dconeill@stanford.edu,
\IEEEauthorrefmark{3}lall@stanford.edu}}

\begin{document}
\maketitle
\begin{abstract}

Consumer Demand Response (DR) is an important research and industry problem, which seeks to categorize, predict and modify consumer\textquotesingle  s energy consumption. Unfortunately, traditional clustering methods have resulted in many hundreds of clusters, with a given consumer often associated with several clusters, making it difficult to classify consumers into stable representative groups and to predict individual energy consumption patterns. In this paper, we present a shape-based approach that better classifies and predicts consumer energy consumption behavior at the household level.  The method is based on Dynamic Time Warping. DTW seeks an optimal alignment between energy consumption patterns reflecting the effect of hidden patterns of regular consumer behavior.
Using real consumer 24-hour load curves from Opower Corporation, our method results in a 50\% reduction in the number of representative groups and an improvement in prediction accuracy measured under DTW distance. We extend the approach to estimate which electrical devices will be used and in which hours.
\end{abstract}
%\newpage
%\tableofcontents
%\newpage

\section{Introduction}
Demand Response (DR) in electrical distribution systems is an important tool for improving a utility\textquotesingle s economic and energy efficiency, reducing emissions, and integrating Renewables. Households represent at least 20\% of electrical energy consumption, making DR an active area of research. In Household DR, individual households are asked to change their energy consumption pattern, often with only a few hours advanced warning. A central authority signals a selected subset of individual households to adjust their power consumption and often incentivizes this behavior through economic means.

The central economic problem in Household DR is selecting the "best" households to incentivize. This entails at least three sub-problems: First, classifying households into groups with similar energy consumption behaviors, second, predicting what a given consumers energy consumption will be in the next day, and third, estimating what devices that consumer uses each day.  The first problem seeks to categorize or segment consumers, and the second and third problems seek to predict the potential benefit of incentivizing a particular customer. 
Unfortunately, a typical consumer exhibits a wide variation in their daily 24 hour energy consumption patterns, both across different days of the week and, most importantly, for a given day of the week. 

Load curve is 24 hour record of the consumer household\textquotesingle s energy consumption that reflects the pattern of household electrical devices usage. Although the daily pattern could be fairly routine, each device usage does not necessarily occur at the exact time every day, resulting in a set of load curves with similarity in shape, but with variation in time. Household load curves classification could be achieved through clustering. One of the challenges is to address similarity measures that are used to make clusters. Similar load curves are grouped together based on calculating similarity among load curves using similarity measures. Since the time of occurrence is not as important in grouping load curves with similar shape, the \textbf{shape-based} approach is appropriate. In the shape based approach, the elastic method such as Dynamic time Warping (DTW) \cite{Bemdt94usingdynamic} is used as a similarity measure \cite{clustering_review}. As a result, the cluster is a grouping of similar shape where load curves in each group do not necessarily have similar value at each time point (Figure \ref{fig:grouping}). 

In this paper we address household load curve clustering and prediction using the shape-based approach.The contributions are as follows.

\begin{enumerate}
\item Using DTW as a dissimilarity measure for clustering results in a 50\% reduction in the number of 24 hour energy clusters, compared to traditional K-means and gaussian based E\&M algorithm.
\item Based on DTW clusters, we develop a Markov model based method for estimating an individual households energy consumption which does not require the knowledge driver variables such as weather data.
\item  We present a new and effective method of estimating when and what devices in a household is used. We show that the probabilistic bound for this fine grain prediction is tight when the power vector is appropriately chosen.
\end{enumerate}

\section{Literature review}
Electricity consumption classification and aggregate load forecasting has been extensively studied. Most forecasting techniques model a relationship between aggregated load demand and driver variables such as calendar effect, weather effect and lagged load demand (e.g. demand at previous hours or at the same hours of previous days). In this paper, however, we focus exclusively on individual household classification and load forecasting. What follows is a summary of prior work on cluster based classification and load forecasting.

Clustering is a technique where similar data are placed into the same group without prior knowledge of groups definition \cite{clustering_review}. There is different similarity measures used to specify the similarity between time series. Among the most popular are Euclidean distance, Hausdorff distance, HMM-based distance, Dynamic time warping (DTW) and Longest Common Sub-Sequence (LCSS) \cite{clustering_review}. One significant application of time series clustering is electricity consumption pattern discovery, where several techniques based on Euclidean distance has been developed.
\cite{mining} uses unsupervised learning technique such as K-means and SOM to classify load curves. In \cite{ram_euclidean}, $L_2$ clustering metric was used, and the resulting number of classes is an order of thousands. In \cite{compare}, results from using various clustering algorithms such as hierarchical clustering, K-means and fuzzy K-means are compared. In \cite{hidden_markov} the authors propose a dynamic model for load curves clustering, fitting individual load curve using semi hidden markov model (SHMM), and classifying those curves by spectral clustering. In \cite{online_media}, the authors present K-SC algorithm, an extension of K-means, with the metric that find the optimal alignment including translation and scaling for matching the shape of the two time series.

Dynamic Time Warping (DTW) distance \cite{Bemdt94usingdynamic} finds the optimal alignment between two time series by stretching or compressing the segments of the series. The distance has been demonstrated in various applications \cite{dtw_word},\cite{dtw_signature},\cite{dtw_genes}. For time series clustering and classification, the authors of \cite{dtw_niennat} examine clustering algorithms using DTW, suggesting the problem with  averaging time series based on this measure. In \cite{dtw_avg}, the author proposes a technique for averaging time series based on DTW distance, which produces clusters with good quality when used in K-means. In \cite{dtw_weight}, the author introduces weighted DTW distance, which penalizes points with higher phase difference between reference point and a testing point to prevent minimum distance distortion. The author in \cite{dtw_deriv} proposes a Derivative Dynamic Time Warping (DDTW), where the distance is measured between higher level feature of the time series obtained by the first derivative. In \cite{dtw_fuzzy}, alternatives fuzzy clustering algorithms using DTW distance are proposed. 

Load forecasting traditionally refers to forecasting the expected electricity demand at aggregated levels, such as small area or utility level \cite{forecast_tutorial} \cite{book_forecast}. In a Global Energy Forecasting Competition 2012 (GEFCom2012), several techniques have been developed for a load forecast at a zonal and system level (sum of electricity consumption from 20 zones) \cite{gefcom}. Among the techniques used by the selected entries are multiple linear regression (MLR), generalized additive model, gradient boosting machines, random forest, neural network etc. Most of the techniques requires a knowledge of driver variable such as temperature and calendar effect to construct a predictor. In \cite{Wang2016}, the authors use the multiple linear regression model with added recency effect. Through variable selection, they find the optimal number of lagged hourly temperature and moving average temperature needed in a regression model. To accommodate some non-linear relationships between load and driver\textquotesingle s variables, generalized additive model has been introduced. In \cite{semi_2010}, the authors develop a semi-parametric model for half hourly demand, and use natural logarithm to transform raw demand. In \cite{semi_2014}, the authors develop the model including long term trend estimate, medium term estimate which describes electricity demand response for temperature changes, and short term estimate which reflects local behavior. In \cite{gradient_boosting}, the authors uses gradient boosting to forecast the load for GEFCom2012. Among several AI techniques, Artificial neural network is of interest for load forecast, since it does not require explicit function of a relationship between input and output \cite{forecast_tutorial}. In \cite{hippert_ann}, the authors compare several ANN based load forecasting techniques.They emphasize the importance of appropriate architecture in neural network design and avoidance of overfitting. Support vector machine (SVM) generally works well data for classification, and can also be used for regression. The authors in \cite{svm} use support vector regression (SVR) with calendar variables, temperature and past load as an input to predict maximum daily values load for UNITE competition 2001. \cite{Hong_2009} applies immune algorithm as a parameter selection for SVR model. However, due to the lack of high frequency temperature data, they also need temperature forecast, which expands the margin for load prediction error. Although most of the existing literature focus on load forecasting at the utility, substation or larger level, the techniques and methodology could be applied to smaller level forecast. Furthermore, the availability of smart meter data from the past decade has enabled hierarchical load forecasting (HLF), which covers forecasting at various levels, from household to corporate level, from few minutes ahead to years ahead. 

%Note that DTW and its variance have been used as a similarity measure for shape-based clustering of the time series. Since grouping the daily electric consumption based on their shape information is desirable, we apply clustering method using DTW distance, and make a prediction based on the clustered prototypes. 

\section{Household Load Curves}
A household load curve is a 24 hour record of the consumer household\textquotesingle s energy consumption. In this paper consumption is sampled every hour using the AMI infrastructure. The load curves reflect the pattern of consumer\textquotesingle s daily lives in the devices they use and the energy they consume: Hair dryer, microwave, range, etc.  However, even though the daily patterns are fairly routine, the exact time or sequence in which devices are used can be expected to vary greatly with the individual\textquotesingle s immediate circumstances (e.g. she blowdries her hair at 6:45 rather than after breakfast at 7:30). This is reflected in the variation of load curves (Figure \ref{fig:same class}), even for a very consistent set of behavioral patterns.This problem makes clustering the data and consumer classification difficult. Clustering can generate many different clusters even for a single individual, greatly complicating individual load prediction. This is especially true for similarity measures such as $L_1$, $L_2$, or other inelastic metrics.  

\begin{figure}[h]
\begin{center}
  \includegraphics[width=.9\linewidth]{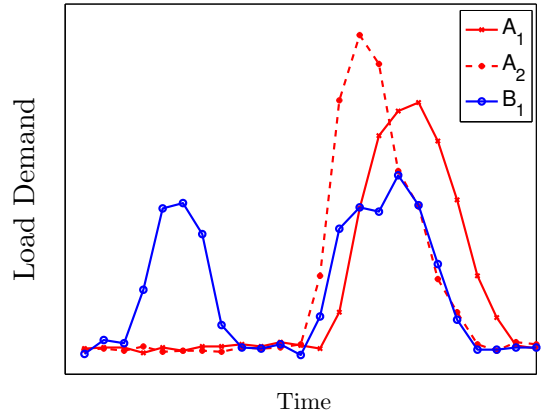}
  \end{center}
  \caption{Three load curves from 2 households. $A_1$ and $A_2$ belongs to household A, and $B_1$ belongs to household B. Clearly, $A_1$ and $A_2$ has a similar shape and should be grouped together. However, using euclidean distance in clustering would not result in $A_1$ and $A_2$ belong to the same cluster since $\|A_1-A_2 \|_2 > \|A_1-A_3 \|_2$ and $\|A_1-A_2 \|_2 > \|A_1-A_2 \|_2$. The shape based similarity measures such as DTW give a lower dissimilarity between $A_1$ and $A_2$ and therefore result in $A_1$ and $A_2$ belong to the same cluster.  }\label{fig:grouping}
  \end{figure}

\begin{figure}[h]
\begin{center}
  \includegraphics[width=.9\linewidth]{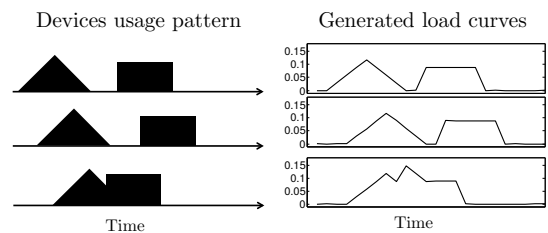}
  \end{center}
  \caption{Three load curves generated from the same pattern of behavior but with variation in time. The triangle represents electric oven usage and the rectangle represents air condition usage. If the top load curve is generated from a routine pattern, the middle load curve is a result from using an oven and air condition 1 hour later than the typical routine. The bottom load curve is a result from  using an oven 1 hour later, and using an air condition 1 hour before the routine.  }\label{fig:same class}
  \end{figure}

\section{Clustering with Dynamic Time Warping (DTW)}
\label{clustering}
Since Household DR is focused on understanding patterns of human electric device usage in order to reduce energy consumption, the shape-based approach clustering is most appropriate for load curves. Time of occurrence of patterns is not as important to find a similarity in shape. Therefore, elastic methods such as Dynamic Time Warping (DTW) \cite{Bemdt94usingdynamic} is a suitable similarity measure. Using DTW, the shapes of two load curves are matched as well as possible by non-linear stretching and contracting of the time axes \cite{clustering_review}. As a result, the clusters can be thought of as a grouping of common electric device usage pattern, where load curves in each group do not necessarily have similar value at each time point.
%When used in clustering, the clusters of load curves with similar patterns are constructed regardless of time points. 

%Assuming that a typical consumer exhibits up to 2 hours variation in the pattern of electrical devices usage, which results in up to 2 hours shift of each hourly record in the corresponding load curve. The \textbf{warping path} \cite{dtw}  of DTW is chosen so that each point in a load curve is compared to at most two previous or two latter points in time from another load curve (figure \ref{fig:dtwpattern}). Using DTW with this warping path as a similarity metric, the clusters of load curves with similar shape are constructed regardless of up to 2 hours difference in time points.

In this work, we assume that a typical consumer exhibits up to about 1 hour variation in time of electrical devices usage, which results in stretching and contracting of the time axes within each 2 hours long period. The \textbf{warping path} \cite{dtw}  of DTW is chosen so that each point in a load curve is compared to at most two previous or two latter points in time from another load curve  (figure \ref{fig:dtwpattern}). Using DTW with this warping path as a similarity metric, the clusters of load curves with similar shape are constructed regardless of stretching and contracting of the time axes.

 To calculate DTW distance between two load curves, $\mbox{DTW}(X,Y)$, we use an algorithm based on dynamic programming to determine the optimal comparison path. The algorithm recursively compute the costs of the alignment between subsequences of $X$ and $Y$ and store them in an accumulate cost matrix $C$. Let $x_n$ and $y_m$ be an hourly energy consumption of different load curves. We use a distance measure $d(x_{n}, y_{m}) = (x_{n} - y_{m})^2$. The cost of the optimal alignment can be recursively computed by  \begin{equation}
C(X_n,Y_m)  = \mbox{min}
\begin{cases}
C(X_{n-1}, Y_{m-1}) &+ \quad d(x_n, y_m) \\
C(X_{n-2}, Y_{m-1})  &+ \quad d(x_{n-1}, y_m)  \\
&+ \quad d(x_n, y_m) \\
C(X_{n-1}, Y_{m-2}) &+ \quad d(x_{n}, y_{m-1})  \\
&+ \quad d(x_n, y_m) \\
C(X_{n-3}, Y_{m-1}) &+ \quad d(x_{n-2}, y_m)  \\
&+ \quad d(x_{n-1}, y_m) \\ 
&+ \quad d(x_n, y_m) \\
C(X_{n-3}, Y_{m-1}) &+ \quad  d(x_{n}, y_{m-2})  \\
&+ \quad d(x_{n}, y_{m-1}) \\
&+ \quad d(x_n, y_m).\\
\end{cases}
 \label{recursion_pattern2}
 \end{equation}
 
where $X_{n}$ is the subsequence $\langle  x_1, \hdots, x_n \rangle$ and $Y_{m}$ is the subsequence $\langle  y_1, \hdots, y_m \rangle$. The overall dissimilarity is given by
Then, \[DTW(X,Y) = C(X_{24},Y_{24}). \]

  \begin{figure}[h]
        \centering
        \includegraphics[width=.5\linewidth]{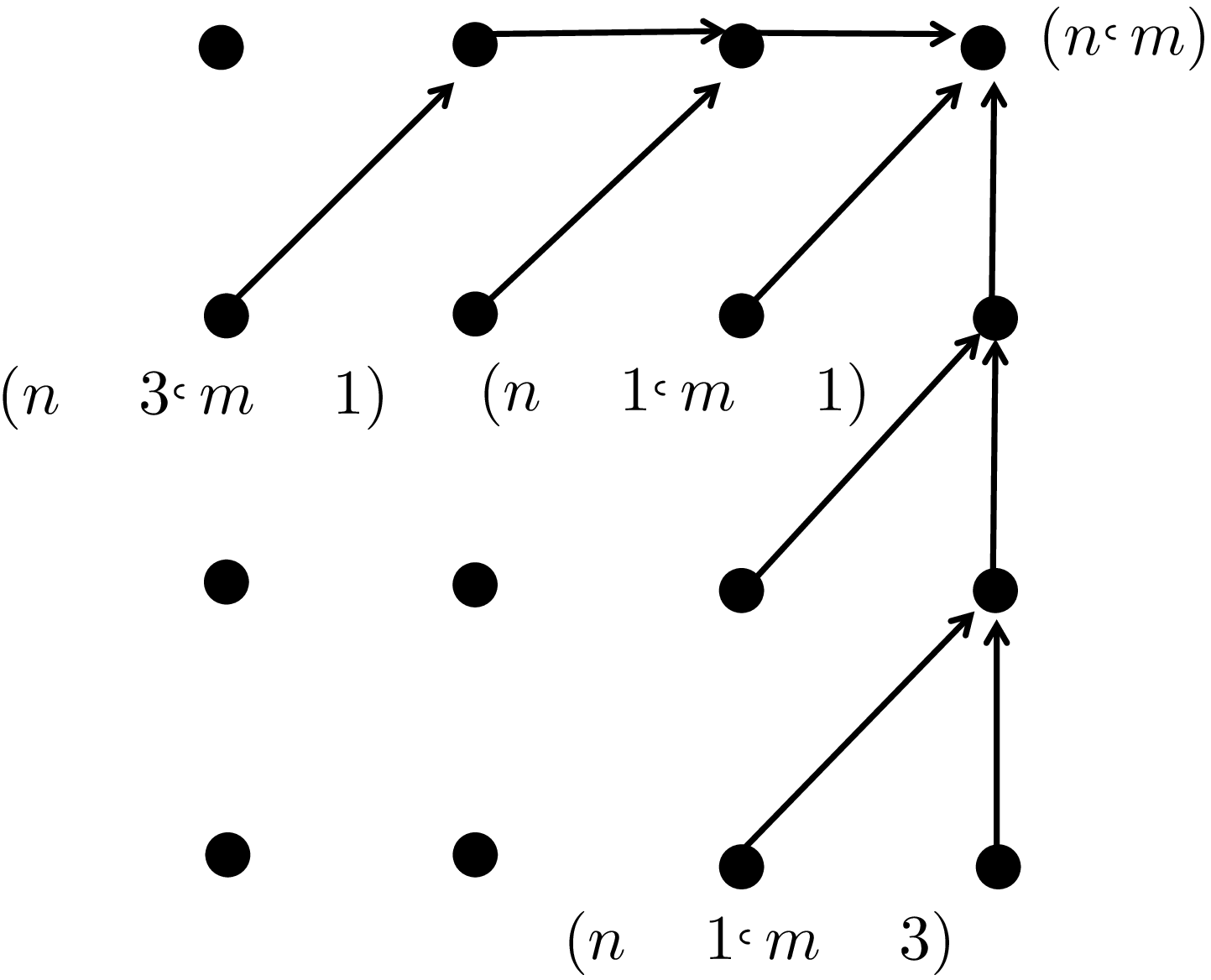}
        \caption{Step size condition in DTW}\label{fig:dtwpattern}
\end{figure}

Given a set of load curves $X$, and the number of clusters $K$, the goal of DTW based clustering algorithm is to find for each cluster $k$, an assignment $C_k$, and the cluster prototype $\mu_k$ that minimize Within Cluster Sum (WC), defined as 
\begin{equation}
  \mbox{WC} = \sum_{k=1}^{K}\sum_{X \in C_k} \mbox{DTW}(X, \mu_k).
  \label{WCeq}
  \end{equation}
 We start the DTW based clustering algorithm with a random initialization of the cluster centers. In the assignment step, we assign each load curve to the closest cluster, based on $\mbox{DTW}(X,Y)$. After finding cluster membership of all load curves, we update cluster prototype, $\mu^*_k$ which is the minimizer of the sum of $\mbox{DTW}(X, \mu_k)$ over all $X \in C_k$ :
\begin{equation}
\mu^*_k = \argmin_{\mu \in C_k}\sum_{X \in C_k} \mbox{DTW}(X, \mu)
\label{avgeq}
  \end{equation}
  
The algorithm is equivalent to the K-medoids clustering \cite{Hastie} using a DTW similarity metric. 

\subsection{Data Description}
The data for clustering consists of 23,254 load curves from 1057 households provided by Opower Corporation. Each load curve is a 24 hour electricity consumption record from one household. Each household (with unique customer id) has total 22 load curves recorded between July 19 2012 to August 9 2012. Each load curve is pre-smoothed using cubic spline and normalized to sum to one. As a baseline, we also cluster the data using K-means and E\&M.

\subsection{Cluster Quality}
As will be seen, DTW results in a smaller number of clusters and lower household variability. To evaluate the quality of clusters, we compute three performance matrices.
\begin{enumerate}
\item The value of the clustering objective function WC, defined in Equation \ref{WCeq}, which measures the compactness of the clusters.
\item The sum of DTW distance between cluster centers, $\mbox{WB} = \frac{1}{2}\sum_{i \neq j}\mbox{DTW}(\mu_i, \mu_j)$, which measures diversity of the clusters \cite{online_media}.
\item The ratio of within cluster sum to between cluster variation (WCBCR) \cite{WCBCR} based on $\mbox{DTW}(X,Y)$, defined by $\mbox{WCBCR} = \frac{WC}{WB}$.
\end{enumerate}

A good clustering has a low value of WC and WCBCR, and large distance between cluster centers (WB). Figure \ref{fig:qualitymeasure} displays WC, WB and WCBCR versus the number of cluster prototypes. For comparison, we also show the results from K-means and E\&M clustering. DTW based clustering achieves lower value of WC and WCBCR, and larger value of WB. 

Equivalently, DTW results in fewer clusters a given the same WC and WCBCR. 
%Figure \ref{fig:showclusters} displays clusters produced from DTW based clustering and K-means, with the same value of WCBCR. 
The appropriate number of clusters could be estimated graphically by using the rule of the `knee' \cite{WCBCR} from WCBCR plot , which gives values between 8 to 10 for DTW clustering, 20 for for K-means and between 14 to 16 for E\&M.

\begin{figure}[t!]

  \begin{subfigure}[t]{0.5\textwidth}
        \centering
         \includegraphics[width=.6\linewidth]{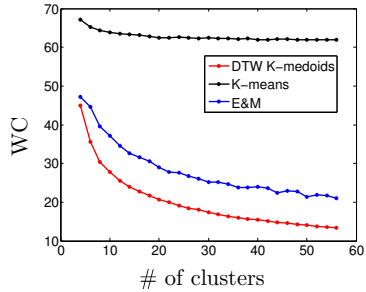}
                \caption{ WC from clustering using DTW is consistently lower than the result using K-means and E\&M.}\label{fig:WC}
    \end{subfigure}
   ~
   \begin{subfigure}[t]{0.5\textwidth}
        \centering
         \includegraphics[width=.6\linewidth]{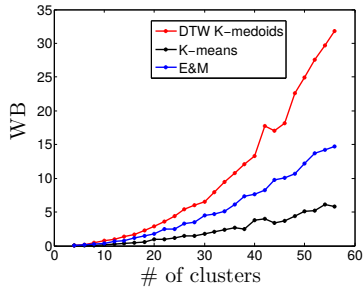}
                \caption{ WB from clustering using DTW is consistently higher than the result using K-means and E\&M.}\label{fig:WB}
    \end{subfigure}
     ~
   \begin{subfigure}[t]{0.5\textwidth}
        \centering
         \includegraphics[width=.6\linewidth]{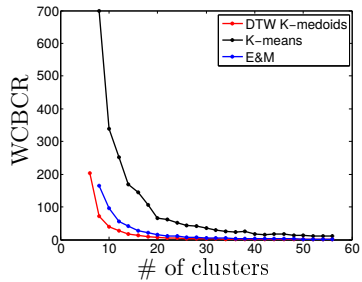}
                \caption{ WCBCR from clustering using DTW is consistently lower than the result using K-means and E\&M.}\label{fig:WCBCR}
    \end{subfigure}%
   \caption{Comparison of adequacy measure of clustering metric} \label{fig:qualitymeasure}
\end{figure}

\subsection{Consumer Variability}
To capture variability in clusters, we utilize a notion of entropy from information theory \cite{ram_euclidean}. 
When $M$ load curves from household $n$ are clustered into $K$ clusters, the entropy is defined as 
\begin{equation}
S_n = - \sum_{k=1}^{K}p_k \mbox{log}_M(p_k)\\
\label{entropy}
 \end{equation}
 where $p_k$ is the ratio between the number of load curve that falls into cluster $k$ and the total number of load curves. Smaller $S_n$ corresponds to greater consistency.

After clustering the load curves into 10 clusters, average entropy from DTW clustering is 0.5, while that from K-means is 0.7 and that from E\&M is 0.64 (Figure \ref{fig:entropy}). The lower average household variability from DTW suggests that energy consumption of each household is more suitably described by an elastic model such as DTW. In Figure \ref{fig:segregate}, weekday load curves for a single individual are considered. Clustering using DTW results in a single cluster, while using K-means with $L_2$ metric results in three clusters. 
   \begin{figure}[h]
\begin{center}
  \includegraphics[width=.75\linewidth]{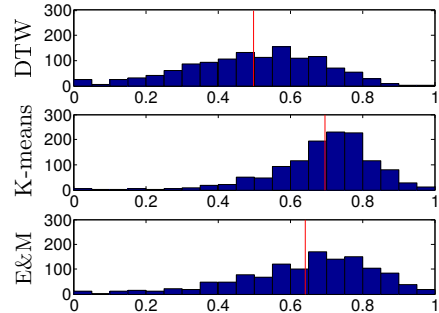}
  \end{center}
  \caption{Entropy distribution after clustering }\label{fig:entropy}
  \end{figure}

  \begin{figure}[h]
\begin{center}
  \includegraphics[width=.75\linewidth]{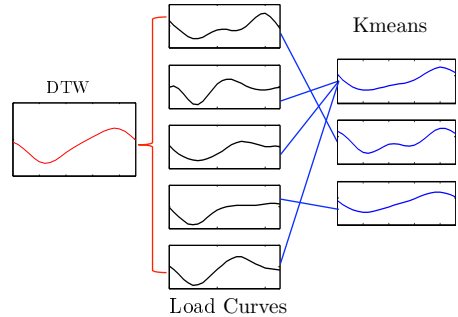}
  \end{center}
  \caption{Classification of load curves for a single individual. While clustering using DTW results in a single cluster, using K-means results in three clusters.}\label{fig:segregate}
  \end{figure}

\section {Load Curve Prediction} 
In this section we make day ahead load curve predictions for an individual household. We use DTW based cluster prototypes and and Markov techniques based on load shapes to make this prediction. The prediction has two steps. First, we select the best next day load shape from the cluster prototypes, conditioned on the current day's load curve. Then, we scale the selected prototype to obtain the prediction.

\subsection{Day Ahead Load Shape Selection}
Conceptually we construct Markov models based on the shapes of load curves, or more usefully, parts of load curves. Given $K$ clusters and $D$ days history of load curves from an individual consumer whose next day consumption to be predicted, let first consider the simplest setting. We encode each past load curve with the index of the nearest cluster under DTW distance. If $a_d$ represents the encoded cluster of day $d$ load curve, we model a sequence $\{a_d: d = 1, \hdots, D \}$ as a stationary Markov process. We use empirical data to calculate the probability transition, $P(a_{d+1}\vert a_{d})$ which is the conditional probability of the encoded cluster of the next day load curve, given that of today load curve. We then use this probability transition matrix to predict the cluster of the next day load curve. The prototype of the selected cluster is then 24 hour long predicted load shape.

More generally, we can apply this concept to portions of 24-hour load curves. We divide each 24 hour day into $n_p$ equal length periods, each period has $\frac{24}{n_p}$ hours long record. We cluster each period\textquotesingle s load curves separately, resulting in $n_p$ sets of K clusters. For load curves of an individual consumer whose next day consumption to be predicted, denote period $p$ of day $d$ load curve by $X_{d,p} \in R^{\frac{24}{n_p}}$. If $\Lambda$ is a set of all possible $X_{d,p}$, and $\Sigma$ is a set of all possible cluster indices $\left(\set{1,\hdots,K}\right)$, the encoder function $\mbox{dp-ENC} : \Lambda \rightarrow \Sigma$ (Algorithm \ref{alg:DTW_assignment}) returns the index of period $p$ cluster whose DTW distance from its prototype to the input is minimum. If $a_{d,p} = \mbox{dp-ENC}(X_{d,p})$, we use empirical data to calculate the probability transition, $P_{p}\left(a_{d,p} \vert \{a_{d,l}\}_{l=1}^{p-1}, \{a_{d-1,l}\}_{l=p}^{n_p} \right)$ which is the conditional probability of the encoded cluster of period $p$, given that of period $p$ on the previous day to period $p-1$ on the same day (Algorithm \ref{alg:markov_pcal}) (Note, the list of conditioning variables can be adjusted as appropriate). We then use this to iteratively predict the cluster and corresponding prototype of the next day load curve (Algorithm \ref{alg:DTW_predictshape}), starting from period 1 to period $n_p$. Figure \ref{fig:DTWmarkovprocess} illustrates this procedure.
Furthermore, to reflect the calendar effect, two Markov based models should be constructed, one for weekdays prediction and the other for weekends prediction.

For example, as illustrated in figure 6, if $n_p = 2$, we would have two sets of K clusters : AM cluster and PM cluster. If $a_{d,AM}$ and $a_{d,PM}$ represents the encoded cluster of $X_{d, AM}$ and $X_{d,PM}$ respectively, we calculate $P_{AM}(a_{d,AM}\vert a_{d-1,AM}, a_{d-1,PM})$ which is the conditional probability of the encoded cluster of the AM load, given that of the previous day AM and PM load. We proceed in a similar fashion to calculate $P_{PM}(a_{d,PM}\vert a_{d,AM}, a_{d-1,PM})$. Table \ref{table:probcluster} and \ref{table:resultcluster} illustrate the calculation.

\floatname{Procedure}{Function}
 \begin{algorithm}
  \caption{DTW cluster encoder for period $p$ day $d$ load curve}
 \label{alg:DTW_assignment}
 \begin{algorithmic}[1]
%\Function{Increment}{$b1$}
 \Function{dp-enc}{$ X_{d,p}, \{\mu_{k,p}\}_{k=1}^K, d, p$}
   \Comment{Assign $p$-cluster to period $p$, day $d$ load curve }
    \State{$\tilde{X}_{d,p} \gets \frac{X_{d,p}}{\sum_{t=1}^{24}X_{d,p}\left[ i\right]}$}
     \State{$a_{d,p} \gets \mbox{argmin}_k \mbox{DTW}(\tilde{X}_{d,p}, \mu_{k,p})$}
 \State \textbf{return} $a_{d,p}$
 \EndFunction
 %\EndProcedure
\end{algorithmic}
\end{algorithm}

 \begin{algorithm}
  \caption{Calculation of probability transition to period $p$ cluster, conditioned on previous periods}
 \label{alg:markov_pcal}
 \begin{algorithmic}[1]
 \Function{p-transition}{ $ \set{ \{a_{d,l}\}_{l=1}^{n_p} }_{d=1}^D, p$ }
 \State{$G = \{1, \hdots, K\}^{n_p + 1}$}\\ \Comment{$G$ is a set of $(n_p + 1)$-plets whose each element takes values from $1$ to $K$}
 \For {$d \mbox{ from } 1 \mbox{ to } D$}
     \State{$v_d \gets \langle \{a_{d,l}\}_{l=1}^{p-1}, \{a_{d-1,l}\}_{l=p}^{n_p} \rangle$}
 \EndFor
  \For {$g \in G$} 
    \If {$\sum_{d=1}^D \mathds{1}\left(v_d = g \right) \neq 0$}
      \For {$i \mbox{ from } 1 \mbox{ to } K$} 
    \State{$ P_{p}(i \vert g) \gets \frac{\sum_{d=1}^D \mathds{1}\left(a_{d,p}=i \right) \mathds{1}\left(v_d = g \right) }{\sum_{d=1}^D \mathds{1}\left(v_d = g \right)}$}
    \EndFor
  \EndIf
  \EndFor
 \State \textbf{return} $P_{p}$
 \EndFunction
\end{algorithmic}
\end{algorithm}
 
  \begin{algorithm}
  \caption{DTW-Markov shape-based prediction }
 \label{alg:DTW_predictshape}
 \begin{algorithmic}[1]
 \Procedure{\mbox{ShapePredict}}{$\set {\{X_{d,p}\}_{d=1}^{D}, \{\mu_{k,p}\}_{k=1}^K\}}_{p=1}^{n_p}$}
    \Comment{Encode period p day d load curve with DTW cluster index}
   \For {$d \mbox{ from } 1 \mbox{ to } D$} 
  \For {$p \mbox{ from } 1 \mbox{ to } n_p$}  
     \State{$ a_{d,p} \gets \Call{\mbox{dp-enc}}{X_{d,p}, \{\mu_{k,p}\}_{k=1}^K, d, p}$}
  \EndFor
  \EndFor \\
   \Comment{Calculate probability transition matrices and select the prototype as a shape predictor for the each period }
           \For {$p \mbox{ from } 1 \mbox{ to } n_p-1$}  
     \State{$ P_{p} \gets \mbox{p-transition}\left( \set{ \{a_{d,l}\}_{l=1}^{n_p}}_{d=1}^D, p\right)$}
     \State{$\hat{a}_{D+1,p} \gets \mbox{argmax}_i P_{p}\left(i \vert  \langle \{\hat{a}_{D+1,l}\}_{l=1}^{p-1}, \{a_{D,l}\}_{l=p}^{n_p}\rangle  \right)$}
      \State{$\hat{S}_{D+1,p}\gets \mu_{\hat{a}_{D+1,p},p}$}
        \EndFor \\
 \EndProcedure
\end{algorithmic}
\end{algorithm}

 \begin{figure}[h]
        \centering
        \includegraphics[width=0.98\linewidth]{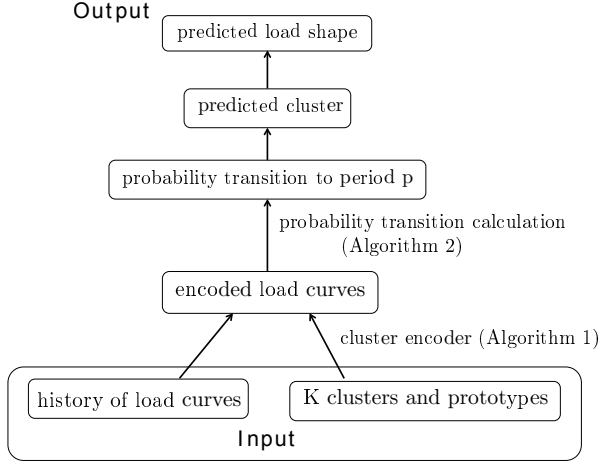}
        \caption{Shape-based prediction procedure (Algorithm \ref{alg:DTW_predictshape})}\label{fig:DTWmarkovprocess}
\end{figure}

\begin{table}[h]
%\resizebox{1\textwidth}{!} {
{
 \begin{tabular}{ | c| c | c | c |c| }
    \hline
    \multicolumn{3}{ |c| }{Encoded cluster} & & \\ 
    \cline{1-3}
    $a_{d-1,AM}$ & $a_{d-1,PM}$ &$a_{d,AM}$ &Count& $P_{AM}(\cdot \vert \cdot)$   \\ \hline
    2 & 1 &  2&9 & 0.75 \\ \hline
     2 & 1 & 3& 3 & 0.25 \\ \hline
      3 & 3 & 3& 3 & 1 \\ \hline
    3 & 1& 2 &2 & 0.33 \\ \hline
    3 & 1 & 3& 4 & 0.67 \\ \hline
  \end{tabular}
  
 \quad 
 \\
 \\
  \begin{tabular}{ | c| c | c |c| c| }
    \hline
    \multicolumn{3}{ |c| }{Encoded cluster} & & \\ 
    \cline{1-3}
   $a_{d-1,PM}$ & $a_{d,AM}$ &$a_{d,PM}$ &Count& $P_{PM}(\cdot \vert \cdot)$\\ \hline
  1 & 2 &  1&11 & 1 \\ \hline
     1 & 3 & 1& 4 & 0.57 \\ \hline
      1 & 3 & 3& 3 & 0.43 \\ \hline
    3 & 3& 1 &3 & 1 \\ \hline
  \end{tabular}
  
}
 \caption{Example of calculating conditional probability of encoded cluster for each period load curve. }
 \label{table:probcluster}
  \end{table}

\begin{table}[h]
\centering
%\resizebox{8cm} {
  % \begin{tabular}{ |p{1cm}| p{1cm}|  p{1cm} |  p{1cm} | }
      \begin{tabular}{ |c| c| c |  c| }
    \hline
    \multicolumn{2}{ |c|}{Previous day encoded cluster} & \multicolumn{2}{ |c| }{Predicted next day cluster}  \\ 
    \hline
    $a_{D,AM}$&$a_{D,PM}$& $\hat{a}_{D+1,AM} $&$ \hat{a}_{D+1,PM} $ \\ \hline
    2 & 1& 2 & 1 \\ \hline
\end{tabular}
    \caption{Example of next day AM and PM cluster prediction using conditional probability from Table \ref{table:probcluster}   }
   \label{table:resultcluster}
  % }
\end{table}

\begin{figure}[t!]

  \begin{subfigure}[t]{0.5\textwidth}
        \centering
         \includegraphics[width=.6\linewidth]{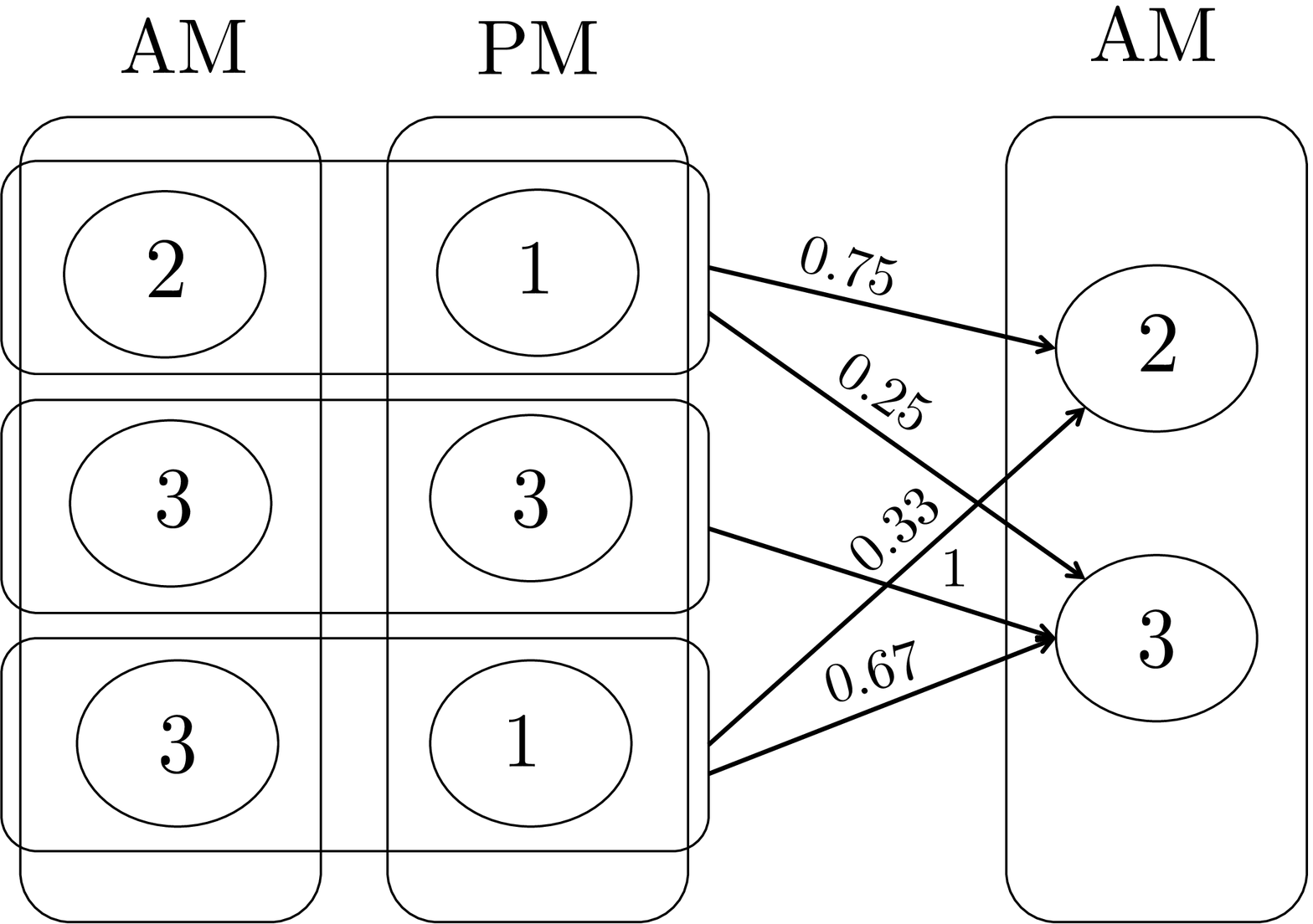}
                \caption{ Probability transition $P_{AM}(a_{d,AM}\vert a_{d-1,AM}, a_{d-1,PM})$ calculated from Table \ref{table:probcluster} }\label{fig:markovchain_AM}
    \end{subfigure}
   ~
   \begin{subfigure}[t]{0.5\textwidth}
        \centering
         \includegraphics[width=.6\linewidth]{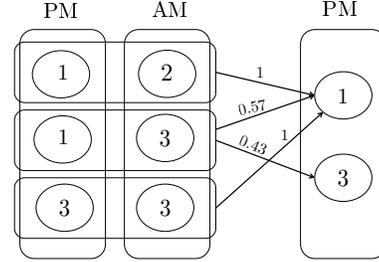}
                \caption{ Probability transition $P_{PM}(a_{d,PM}\vert a_{d-1,PM}, a_{d,AM})$ calculated from Table \ref{table:probcluster} }\label{fig:markovchain_PM}
    \end{subfigure}%
        \caption{Markov chain calculated from Table \ref{table:probcluster} } \label{fig:markovchain}
\end{figure}

\subsection{Adjustment on Predicted Load Shape}
We obtain from  load shape prediction procedure (Algorithm \ref{alg:DTW_predictshape}) the prototypes for each portion (period) of the next day load curve. However, these prototypes have been normalized from the clustering process. Therefore, the final forecast of each period is obtained by an extra scaling step,
\begin{align}
\hat{X}_{D+1,p} &= \alpha_{D+1,p}\hat{S}_{D+1,p} \label{scaling}.
\end{align}
The scaling coefficient $\alpha_{D+1,p}$ is computed to minimize the weighted sum of squared errors of the past forecast,
\begin{align}
\alpha_{D+1,p} &= \mbox{argmin}_{\alpha}\sum_{d=D+1-M}^D \| \alpha \hat{S}_{d,p} - X_{d,p} \|_2^2  \\
                         &= \frac {\sum_{d=D+1-M}^D \langle \hat{S}_{d,p}, X_{d,p} \rangle}{\sum_{d=D+1-M}^D \| \hat{S}_{d,p}\|_2^2},
\end{align}
where $X_{d,p}$ is an actual period $p$ day $d$ load curve, $\hat{S}_{d,p}$ is a corresponding predicted prototype, $M$ is the number of previous days for which load curve predictions have been made, and $\beta $ is a weighting factor in the range $0 < \beta_i < 1$ whose effect is to de-emphasize the old data.

However, if the past load forecast is not available (for example, when the number of available past load information is small and hence the recursive prediction would not be accurate), we could obtain a naive estimate,
\begin{align}
\hat{X}_{D+1,p} &= \mbox{argmin}_{\alpha}\sum_{d=1}^D \| \alpha \langle \hat{S}_{D+1,p} ,\mathbf{1}\rangle \mathbf{1} - X_{d,p}  \|_2^2 \\
    &= \frac {\sum_{d=1}^D \langle X_{d,p}, \mathbf{1} \rangle}{D \langle \hat{S}_{D+1,p}, \mathbf{1} \rangle}
\end{align}
where $\mathbf{1}$ is an all-ones vector. Using this estimate, the sum of the final predicted load curve is the average of past period $p$ load curves sum. 

%\begin{algorithm}
%  \caption{Predicted load shape scaling}
% \label{alg:DTW_predictload}
% \begin{algorithmic}[1]
% \Procedure{\mbox{Loadscaling}}{$\set{ \hat{S}_{D+1,p}, \alpha_{D+1,p}}_{p=1}^{n_p} $}
% \State{$w \gets \frac{24}{n_p}$}
%    \For {$p \mbox{ from } 1 \mbox{ to } n_p-1$}  
%     \State{$\hat{X}\left[(p-1)w+1:pw\right] \gets \alpha_{D+1,p}\hat{S}_{D+1,p}$}        
%     \EndFor 
% \EndProcedure
%\end{algorithmic}
%\end{algorithm}

\subsection{Load Curve Prediction Experiment}
In this experiment we predict next day load curves using empirical data provided by Opower.  Based on clusters from section \ref{clustering}, we encode the load curves and construct two Markov based models: one for weekdays and one for weekends. We divide households into two groups: validation set for prediction model parameters selection and the test set for prediction. The validation set is 3300 load curves recorded between July 19 2012 to August 9 2012. from 150 households in low variability group. Each household has 22 consecutive days of data. We use "leave one out" cross validation by using 21 days to construct Markov based models, and make the prediction for the other day. The test set is 1100 load curves from another 50 households, also in low variability group. 

\subsection{Prediction Model Parameters Selection}
We need to specify number of training set clusters ($K$) and number of periods ($n_p$) for our DTW-Markov prediction model. We select these parameters using trial-and-error method, varying $K$ and $n_p$, and selecting the pair that results in the lowest normalized DTW error (DTWE) on validation data, where
\[ \mbox{DTWE}(\hat{x}, x) = \sqrt{\frac{\mbox{DTW}(\hat{x}, x)}{\| x\|_2^2}}.\]
The reason for using DTWE as opposed to widely used error measure such as mean absolute percentage error (MAPE) is to evaluate the predicted sequence based on how well it describes a consumer underlying electricity usage pattern. 

We vary $K$ from 2 to 50, incrementing by 2 and vary $n_p$ from 1 to 3. There are total 75 models. Table \ref{table:DTWE} enumerates the models based on different $K-n_p$ pairs where each model applied to all households, we obtain the table corresponding to the mean DTWE values on the validation data. As $K$ increases, DTWE improves up to critical number of cluster. After this value, the behavior gradually stabilizes. The number of clusters where DTWE begins to stabilize is around 10 for $n_p = 1$ and around 12 for $n_p=2$. We choose the model with $n_p = 2$ and $K=12$ to balance the tradeoff between DTWE and number of clusters, which results in DTWE of 0.172 on validation data.  

%\begin{table}[h]
%\centering
%\pgfplotstabletypeset[
%    color cells={min=0.169,max=0.242},
%    col sep=comma,
%    /pgfplots/colormap={whiteblue}{rgb255(0cm)=(255,255,255);
%      rgb255(1cm)=(188,0,0)},
%      precision=3,
%       columns/K/.style={reset styles},
%]{
% K, $n_p=1$,$n_p=2$,$n_p=3$
%2, 0.242, 0.222, 0.210 
%4, 0.222, 0.198, 0.206
%6, 0.212, 0.181, 0.202
%8, 0.208, 0.178, 0.201
%10, 0.202, 0.176, 0.200
%12, 0.200, 0.172, 0.200
%14, 0.200, 0.172, 0.200
%16, 0.201, 0.169, 0.197
%18, 0.200, 0.170, 0.199
%20, 0.198, 0.169, 0.198
%22, 0.199, 0.170, 0.201
%24, 0.197, 0.169, 0.202
%26, 0.197, 0.169, 0.202
%}
%\caption{ Mean DTWE value from the prediction on the validation data (leave one out cross validation) using DTW-Markov based model with varying $K-n_p$ pairs. We only shows the result with K up to 26 to avoid redundancy.  }
%\label{table:DTWE}
%
%\end{table}

\begin{table}[h]
\centering
      \begin{tabular}{ |c| c| c |  c| }
    \hline
      $K$&$n_p = 1$& $n_p = 2$&$ n_p = 3$ \\ \hline
2& 0.242 &0.222 &0.210  \\ \hline
4& 0.222& 0.198& 0.206 \\ \hline
6& 0.212& 0.181&0.202 \\ \hline
8& 0.208& 0.178& 0.201 \\ \hline
10& 0.202& 0.176& 0.200 \\ \hline
12& 0.200& 0.172& 0.200 \\ \hline
14& 0.200& 0.172& 0.200 \\ \hline
16& 0.201& 0.169& 0.197 \\ \hline
18& 0.200& 0.170& 0.199 \\ \hline
20& 0.198& 0.169& 0.198 \\ \hline
22& 0.199& 0.170&0.201 \\ \hline
24& 0.197& 0.169& 0.202 \\ \hline
26& 0.197& 0.169& 0.202 \\ \hline
\end{tabular}
    \caption{Mean DTWE value from the prediction on the validation data (leave one out cross validation) using DTW-Markov based model with varying $K-n_p$ pairs. We only shows the result with K up to 26 to avoid redundancy.   }
   \label{table:table:DTWE}
  % }
\end{table}

\subsection{Prediction Results}
Using the DTW-Markov based model ($n_p = 2$, $K=12$) to make a prediction on test data (Electricity consumption on August $10^{th}$ 2012 for each of 50 Opower customer households), results in mean DTWE of 0.182. Table \ref{table:DTWE_compare} compares our prediction with the some existing techniques in the literature.

\begin{table}[h]
\centering
   \begin{tabular}{|K{2cm} |c| K{4.5cm}|}
    \hline
    Prediction Model&DTWE&Description \\ \hline
     DTW-Markov &0.182&  DTW-markov model with $n_p = 2$, $K=12$ \\ \hline
    Tao\textquotesingle s vanilla with recency effect (A)  & 0.214& Multiple linear regression model with added recency effect \cite{Wang2016} (The original version of this model \cite{Hong_2010} is used to produce benchmark scores for GEFCom2012) \\ \hline
Semi-parametric additive model (B)   & 0.249& ST model (short term forecast) proposed by \cite{semi_2014} \\ \hline
    Support Vector machine (C) & 0.203& Support vector regression (SVR) (This was the winning entry in 2001 EUNITE competition in daily maximum load forecast\cite{svm}. For comparison purpose, we apply SVR to forecast hourly load.) \\ \hline
    Artificial Neural Network (D) & 0.264&  ANNSTLF \cite{ANNSTLF}, an ANN based model for short term load forecast \\ \hline
   
\end{tabular}
    \caption{Mean DTWE value from the prediction on the test data from DTW-Markov model ($n_p = 2$, $K=12$), compared to the selected models from the lieterature }
   \label{table:DTWE_compare}
\end{table}

DTW-Markov model reduces the DTWE by 15\%, 27\%, 10\% and 31\% from model A, B, C and D respectively . Figure \ref{fig:dtw_markov_result} shows predicted load curve from DTW-Markov model, compared to those from model A-D.
\begin{figure}[t!]
  \begin{subfigure}[t]{0.5\textwidth}
        \centering
         \includegraphics[width=0.5\linewidth]{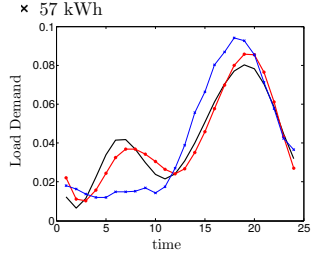}
                \caption{ True and predicted load curve from DTW-Markov and A}\label{fig:predictresult1}
    \end{subfigure}
  ~

   \begin{subfigure}[t]{0.5\textwidth}
        \centering
         \includegraphics[width=0.5\linewidth]{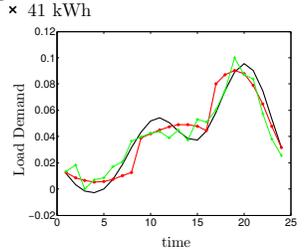}
                \caption{ True and predicted load curve from DTW-Markov and B}\label{fig:predictresult3}
                \end{subfigure}
        ~
           \begin{subfigure}[t]{0.5\textwidth}
        \centering
         \includegraphics[width=0.5\linewidth]{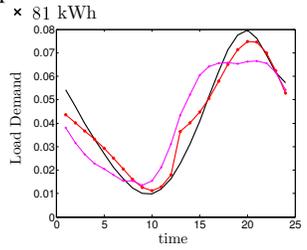}
            \caption{ True and predicted load curve from DTW-Markov and C}\label{fig:predictresult2}
    \end{subfigure} %
     ~
   \begin{subfigure}[t]{0.5\textwidth}
        \centering
         \includegraphics[width=0.5\linewidth]{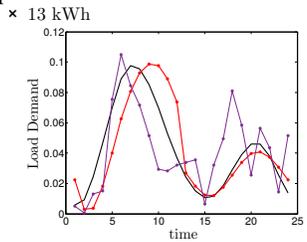}
               \caption{ True and predicted load curve from DTW-Markov and D}\label{fig:predictresult4}
    \end{subfigure}
    
      \caption{True and predicted load curves of households from prediction test data. Black : True load curve, Red : DTW-markov model, Blue : model A , Green : model B, Magenta : model C, Purple : model D} \label{fig:dtw_markov_result}
\end{figure}

\subsection{Complexity Analysis}
The complexity of DTW-Markov prediction algorithm occurs from clustering and Markov chain modeling. We consider the complexity in each task.
\begin{enumerate}
\item{K-medoids with DTW distance clustering}\\
K-medoids requires calculation of DTW distance between pairs of load curves in the training set. Given that training data consists of $N_{tr}$ load curves, the complexity of PAM (Partitioning Around Medoids) is $O(K(N_{tr}-K)^2)$ for each iteration. We need to vary $K$ for model selection. However, since the approximated number of cluster of full load curve from DTW  is around 10 (section \ref{clustering}), we limit $K$ to at most 50. As a result, the complexity of this step is simply $O(N_{tr}^2)$. Furthermore, we only need to do clustering once. The set of prototypes could be used to predict 24 hours ahead consumption of any consumer. 

\item{Markov Chain Modeling}\\
To predict an electricity consumption of a household on day $D+1$, given past $D$ days load curves, we need to model a Markov chain from \{$a_{d,p}$, $d = 1, \hdots, D$, $p = 1, \hdots n_p$\}. Cluster assignment step requires calculation of DTW distance between each past day load curve and the cluster prototypes, resulting in a complexity of $O(n_pKD)$. Probability transition calculation requires at most $O(K^{n_p+1}D)$ operations However, with the bounded value of $n_p$ and $K$, the complexity becomes $O(D)$.

\end{enumerate}

\section{Power Level Decomposition (PLD)}
In this section, we extend the prediction problem to estimating how much energy each device uses each hour during a 24 hour period. For example, estimating air-condition energy usage when only total energy usage from all devices in the house is observed. We consider two cases. One, where an entire 24 hour load curve is available to estimate device energy usage. Second, where only previous days load curves are available and an estimate is needed for the next day. 

\subsection{PLD estimation from full load curve information}
For each load curve, we model an Appliance Usage Matrix, $\tilde{A}$ where $\tilde{A}_{ij}$ indicates the amount of energy used at power level $j$ (corresponding to device $j$) divided by power level $j$, during hour $i$. 

We first define the following matrices involved in this section.
\begin{enumerate}
\item $\tilde{A}$ is the actual but unknown Appliance Usage Matrix.
\item $A$ is an estimate of $\tilde{A}$, obtained from solving \ref{minnormA}.
\item $\hat{A}$ is a predicted Appliance Usage Matrix, based on partially known load curve.
\end{enumerate}

Since $\tilde{A}$ is not directly observable, we first indicate how to find $A$ given complete load curve information. There are several approaches to estimate $A$ including $L_1$ norm minimization which promotes sparsity, total variation (TV) norm minimization, and other forms of convex optimization. However, we first introduce a minimum Frobenius norm estimate, which has useful properties for clustering and prediction problem.
\begin{equation}
 A = \argmin_{A \alpha p = x}{\|A\|_F},\\
\label{minnormA}
\end{equation}
whose closed form solution is 
\begin{equation}
A = \frac{1}{\alpha \|p\|^2}xp^T\\
\label{Asolution}
\end{equation}
where $x$ is complete 24 hour vector load curve, $p$ is a vector of power levels corresponding to the operation of different devices, and $\alpha$ is an arbitrary positive scaling for $p$ vector, reflecting different units of power. Note that $A_{ij}p_j$ is the energy usage of device $j$ during hour $i$.
Figure \ref{fig:A_1_norm1} and \ref{fig:A_4_norm1} shows the sparsity pattern of PLD matrices when the cutoff is its median. The rows of $A$ with larger elements corresponds to the hours with higher energy usage. 

\begin{figure}[t!]
    \centering
    \begin{subfigure}[t]{0.4\textwidth}
        \centering
        \includegraphics[width=.6\linewidth]{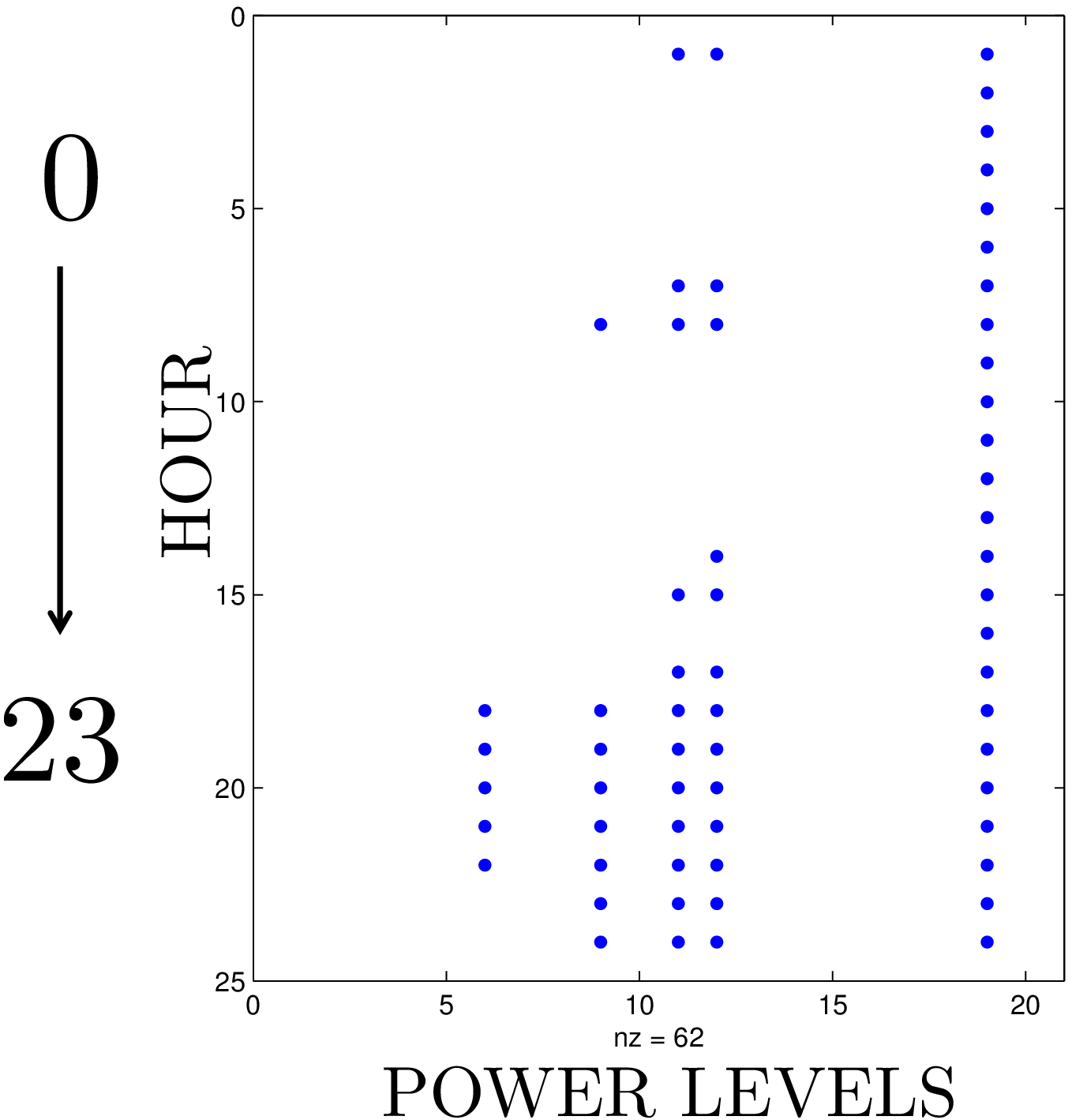}
        \caption{ Sparsity pattern of an Evening Peak)}\label{fig:A_1_norm1}

    \end{subfigure}
    ~ 
    \begin{subfigure}[t]{0.4\textwidth}
        \centering
         \includegraphics[width=.6\linewidth]{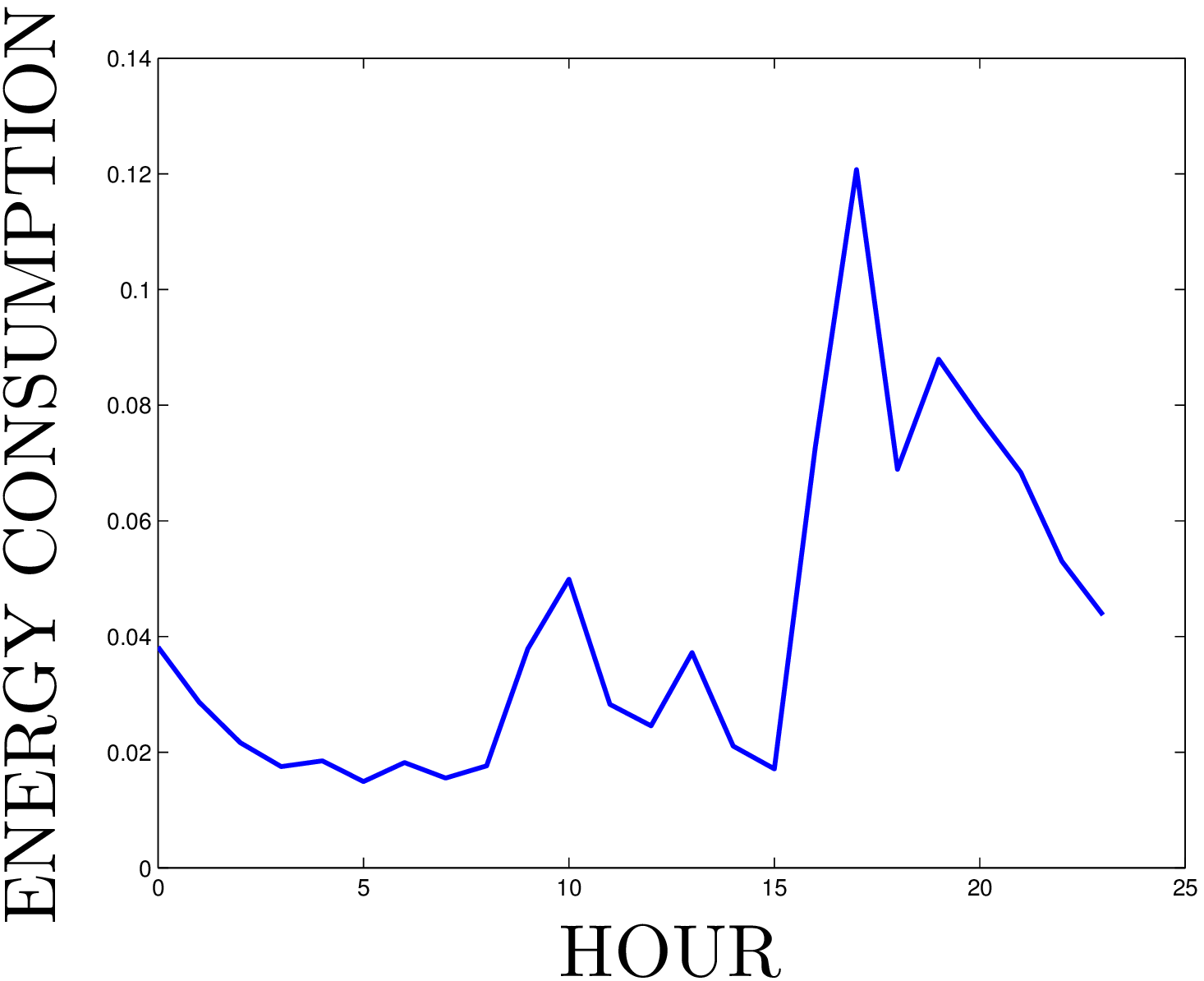}
                \caption{Load curve of an Evening Peak}\label{fig:X1}
    \end{subfigure}
    \caption{Sparsity pattern of the PLD matrix and the load curve of an Evening peak consumer}
\end{figure}

\begin{figure}[t!]
    \centering
    \begin{subfigure}[t]{0.4\textwidth}
        \centering
        \includegraphics[width=.6\linewidth]{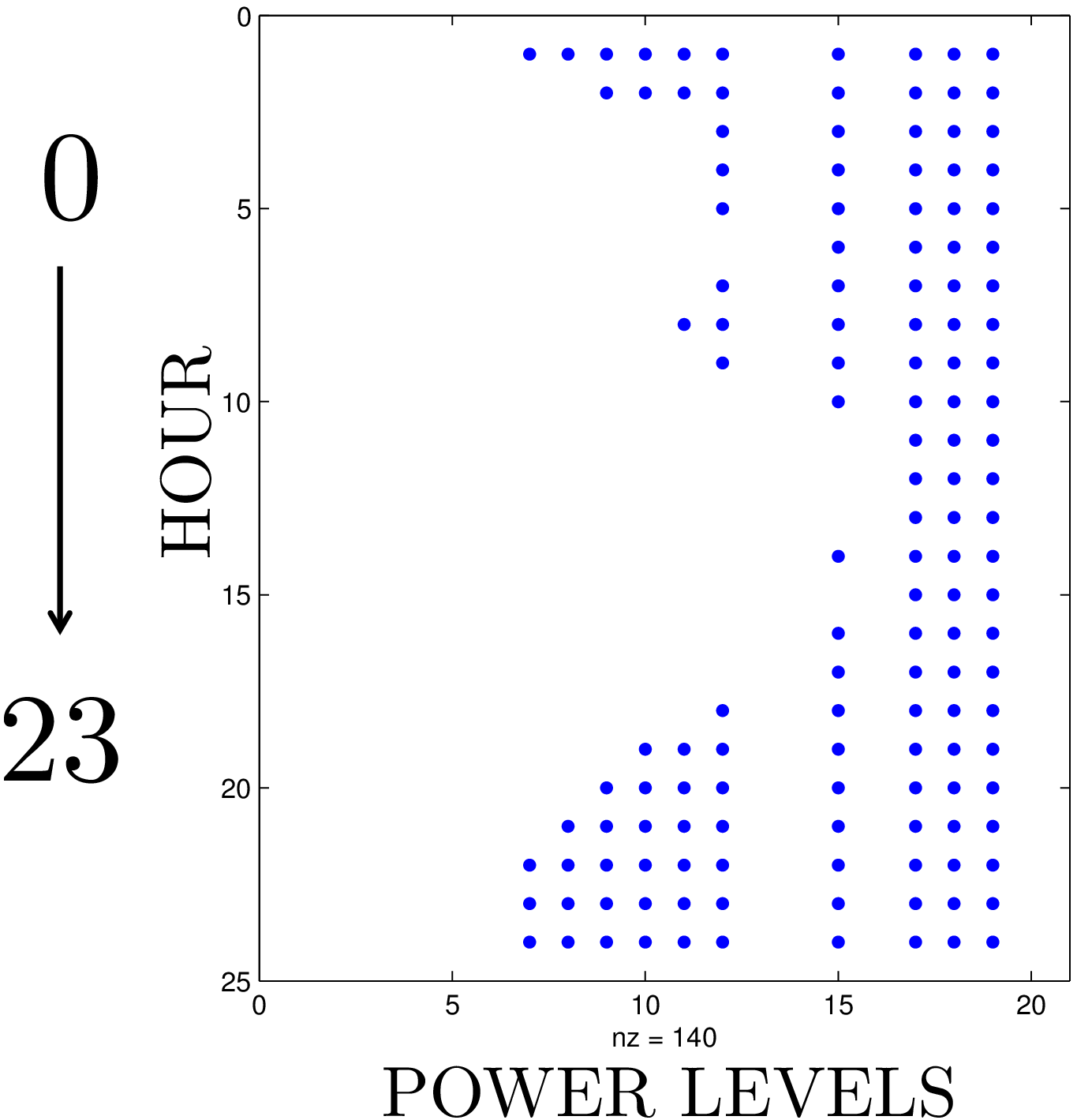}
        \caption{ Sparsity pattern of a Dual Morning \& Evening Peak)}\label{fig:A_4_norm1}

    \end{subfigure}
    ~ 
    \begin{subfigure}[t]{0.4\textwidth}
        \centering
         \includegraphics[width=.6\linewidth]{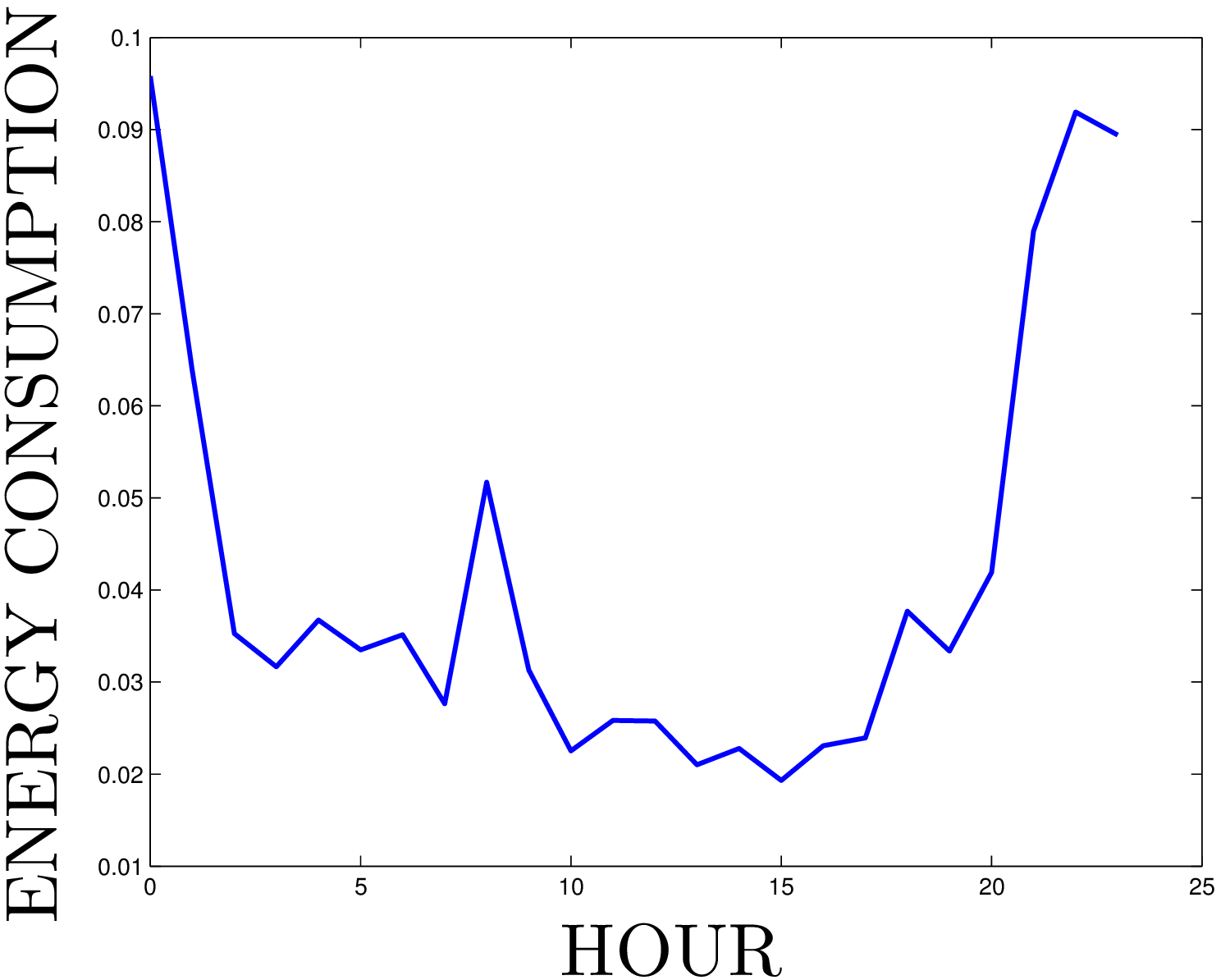}
                \caption{Load curve of a Dual Morning \& Evening Peak}\label{fig:X4}
    \end{subfigure}
    \caption{Sparsity pattern of the PLD matrix and the load curve of a Dual Morning \& Evening Peak consumer}
\end{figure}

\subsection{Clustering PLD matrices}
Let $A_x$ and $A_y$ be PLD matrices corresponding to load curves $x$ and $y$ obtained from the optimization \ref{minnormA}. Define DTW distance between two PLD matrices as $\mbox{DTW}(A_x, A_y) = \sum_{i} \mbox{dtw}(a_{xj}, a_{yj}), $ where $a_{xj}$ and $a_{yj}$ are $j^{th}$ column of $A_x$ and $A_y$ respectively. This could be interpret as a sum of DTW distance across time over all power levels. We have the following equivalences.
\begin{align}
\| A_x - A_y\|_F^2 &=\frac {\| x - y\|_2^2}{ \|\alpha p\|_2^2}\label{l2_equivalent}\\
\mbox{DTW}(A_x, A_y) & = \frac{\mbox{dtw}(x,y)}{ \|\alpha p\|_2^2} \label{DTW_equivalent}
\end{align}
As a result, clustering PLD matrices using K-means or K-medoids with DTW dissimilarity is equivalent to clustering the load curves and then converting each load curve in the clusters to PLD using equation \ref{Asolution}.

\subsection{PLD Prediction}\label{PLDpred}
In this section, we estimate 24 hour ahead device usage from previous days usage information. Given previous days load curve, the idea is to estimate $\tilde{A}$. 

Given previous days usage information, load curve prediction $\hat{x}$ is obtained from DTW-markov procedure (Algorithm \ref{alg:DTW_predictshape}, equation \ref{scaling}). Then, PLD prediction $\hat{A}$ is obtained from $\hat{x}$ using \ref{minnormA}. 
Using the Frobenius norm distance measure, PLD prediction error is
\begin{align*}
 \mbox{Prediction Error : MSE}(\hat{A}, \tilde{A}) & =  \frac{\|\hat{A}-\tilde{A}\|_F}{\| \tilde{A} \|_F}
\end{align*}

Let $\tilde{A} = A + H$, where $H$ has rank $R_H$ and $\sigma_{1}$ is a maximum singular value of $H$, we have the following bound on the PLD prediction error,
\begin{align}
\label{theorem1}
\begin{split}
\sqrt{\frac{\|\hat{x} - x\|_2^2  + \| \alpha p\|_2^2\sigma_{1}^2}{\| x\|_2^2 + \| \alpha p\|_2^2R_H\sigma_{1}^2}} &\leq \frac{\|\hat{A}-\tilde{A}\|_F}{\| \tilde{A} \|_F}  \\
&\leq \sqrt{\frac{\|\hat{x} - x\|_2^2  + \| \alpha p\|_2^2R_H\sigma_{1}^2}{\| x\|_2^2 + \| \alpha p\|_2^2\sigma_{1}^2}}.
\end{split}
\end{align}

We find the CDF of upper and lower bounds of PLD prediction error under Frobenius norm distance measure, which are random functions of $\sigma_1(H)$. We model $H$ as a gaussian random matrix \cite{johnstone2001}, whose entries are independent standard Gaussian variates. Equivalently, $\sigma_1^2$ is the largest eigenvalue of the real Wishart matrix $H^TH$. The distribution of $\sigma_1^2$  approaches the Tracey-Widom law of order 1, which can be approximated by a properly scaled and shifted gamma distribution \cite{Chiani201469}. 

For the upper bound $g_U(\sigma_1^2) = \sqrt{\frac{\|\hat{x} - x\|_2^2  + \| \alpha p\|_2^2R_H\sigma_{1}^2}{\| x\|_2^2 +  \| \alpha p\|_2^2\sigma_{1}^2}}$, the CDF

\begin{equation}
F_{g_U}(t) \\
 =
\begin{cases}
F_{\sigma_1^2} \left( \frac{t^2 \|x\|_2^2 - \|x - \hat{x}\|_2^2}{\| \alpha p\|_2^2(R_H - t^2)} \right) & \mbox{if} \quad t^2 \leq R_H \\
1 - F_{\sigma_1^2} \left( \frac{t^2 \|x\|_2^2 - \|x - \hat{x}\|_2^2}{\| \alpha p\|_2^2(R_H - t^2)} \right)= 1 &\mbox{otherwise}.\\
\end{cases}
 \label{PLDupper_prod}
 \end{equation}

For the lower bound $g_L(\sigma_1^2) = \sqrt{\frac{\|\hat{x} - x\|_2^2  + \| \alpha p\|_2^2\sigma_{1}^2}{\| x\|_2^2 + \|\alpha p\|_2^2R_H\sigma_{1}^2}}$, the CDF
\begin{equation}
F_{g_L}(t) \\
 =
\begin{cases}
F_{\sigma_1^2} \left( \frac{\|x - \hat{x}\|_2^2 - t^2 \|x\|_2^2}{\| \alpha p\|_2^2(t^2R_H - 1)} \right) = 0 & \mbox{if} \quad t^2 \leq \frac{1}{R_H} \\
1 -F_{\sigma_1^2} \left( \frac{\|x - \hat{x}\|_2^2 - t^2 \|x\|_2^2}{\| \alpha p\|_2^2(t^2R_H - 1)} \right) & \mbox{otherwise},\\
\end{cases}
 \label{PLDlower_prod}
 \end{equation}
 
where we assume that load curve prediction error under $L_2$ distance is bounded ($ \frac{1}{R_H} \leq \frac{\|x-\hat{x}\|_2^2}{\|x\|_2^2} \leq R_H$). This is a valid assumption since $R_H$ tends to be very large. For example, the number of electricity devices in the household (number of columns of $H$) is typically larger than 10, and the number of electricity consumption records in a day (number of rows of $H$ which is 24 for our paper) could be as large as 144, for smart meter data collected at 10 minutes interval. 

We are interested in how scaling factor $\alpha$ affects the CDF of these bounds. We plot the CDF of lower and upper bounds under different values of $\alpha$ and make the following observations: \\
\begin{enumerate}
\item With $\alpha$ large, $g_U \approx \sqrt{R_H}$, and $g_L \approx \sqrt{\frac{1}{R_H}}$. Therefore, $F_{g_U}(t) \approx 0$ when $t^2 \leq R_H$, and $F_{g_L}(t) \approx 1$ when $t^2 \geq \frac{1}{R_H}$ (Figure \ref{fig:Abound_largep}). \\
\item  With smaller $\alpha$, the gap between CDF of both bounds is smaller (Figure \ref{fig:Abound_medp}). \\
\item With $\alpha$ very small, $g_U \approx g_L \approx \frac{\| \hat{x} - x\|_2}{\|x\|_2}$. Therefore, the CDF of both bounds has a sharp cutoff at $t =\frac{\| \hat{x} - x\|_2}{\|x\|_2} $ (Figure \ref{fig:Abound_smallp}). 
\end{enumerate}

 \begin{figure}[!htbp]
    \begin{subfigure}[t]{0.45\textwidth}
        \centering
         \includegraphics[width=.6\linewidth]{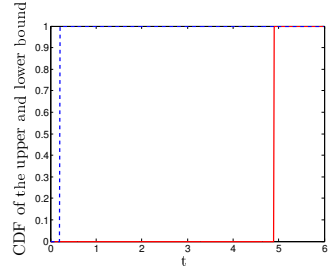}
                \caption{CDF of the upper and lower bound when $ \| \alpha p\|_2^2 = 9455$}\label{fig:Abound_largep}
    \end{subfigure}
   ~
    \begin{subfigure}[t]{0.45\textwidth}
        \centering
        \includegraphics[width=.6\linewidth]{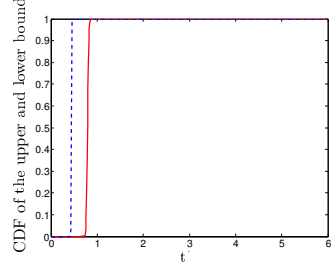}
        \caption{ CDF of the upper and lower bound when $ \| \alpha p\|_2^2 = 9.455 \times 10^{-6}$}\label{fig:Abound_medp}
    \end{subfigure}
    ~
    \begin{subfigure}[t]{0.45\textwidth}
        \centering
        \includegraphics[width=.6\linewidth]{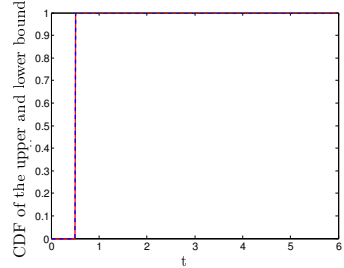}
        \caption{ CDF of the upper and lower bound when $ \| \alpha p\|_2^2 = 9.455 \times 10^{-8}$}\label{fig:Abound_smallp}
    \end{subfigure}%
    \caption{CDF of the upper and lower bound of PLD prediction error when $R_H = 24$, $\|x\|_2^2 = 0.0571 $, and $\| x - \hat{x}\|_2^2 = 0.0144 $ }
\end{figure}
\FloatBarrier

 The bound of the PLD prediction error is tight when the units of the power vector are appropriately chosen. As $\|\alpha p\|_2$ is smaller, both bounds approaches $\frac{\|x-\hat{x}\|_2^2}{\|x\|_2^2}$, which is load curve prediction error under $L_2$ distance.
 
\subsection{Sparse PLD Estimate}
To avoid unnecessary zero elements in the estimate of $\tilde{A}$, we could use $L_1$ minimization or a combination of both Frobenius norm and $L_1$ norm regularization. Define an entry-wise norm of $m \times n$ matrix as $\|A\|_1 = \left( \sum_{i=1}^m \sum_{j=1}^n \lvert a_{ij}\rvert\right)$, then the estimate could be obtained by solving

\begin{equation}
 A = \argmin_{A \alpha p = x}{\lambda_1\|A\|_F + \lambda_2\|A\|_1}. \\
\label{minnormA2}
\end{equation}

The larger $\lambda_2$ relative to $\lambda_1$, the estimate tends to be more sparse. However, when $\lambda_2 \neq 0$, $A$ does not have a closed form solution. As a result, to make a prediction for each device usage pattern (each column), we need to apply DTW K-medoids clustering and prediction procedure (Algorithm \ref{alg:DTW_predictshape}) to each column separately. Furthermore, the equivalence relations \eqref{l2_equivalent}, \eqref{DTW_equivalent} and the prediction error bound \eqref{theorem1} are no longer true for this case.

\section{Conclusion}
We present the shape-based approach to household level electric load curve clustering, household level load curve prediction, and time of use device estimation. Based on the observation that individuals express their patterns of energy consumption behavior at different times in different days, we present a new DR clustering metric, Dynamic Time Warping (DTW) for shape-based clustering. Using DTW, the shapes of two load curves are matched as well as possible by non-linear stretching and contracting of the time axes. As a result, clustering based on DTW metric results in a relatively smaller number of classes, higher clustering efficiency measures, and smaller household variability compared to traditional K-means and gaussian based E\&M algorithm.

Based on the prototypes from DTW based clustering, we encode the load curves with the nearest prototype under DTW distance, and construct Markov based models on the sequence of encoded load curves. The prediction is the best next day prototype, conditioned on the current day's encoded load curve. To measure the accuracy of our prediction method, we define the prediction error metric under DTW distance, DTWE. Our prediction method results in a lower prediction error on average, compared to selected forecasting techniques in the literature. 

Furthermore, we introduce Power Level Decomposition (PLD), where a load curve is decomposed into a matrix which provides an information on energy each device uses each hour during a 24 hour period. Given previous days load curves, we propose the method to predict the full PLD matrix. The probabilistic bound for this fine grain prediction error under Frobenius norm is tight when the power vector is appropriately chosen. In the future, we would like to explore and compare the performance of PLD estimates from different kinds of regularization or convex optimization. 

 \section{Appendix}
 \subsection{Prediction Model (A-D) Implementation }
 Here we provide the detailed implementation on load curve prediction models we use to compare with our DTW-Markov based model (Table \ref{table:DTWE_compare}).
 \begin{enumerate}
\item Model A : Tao\textquotesingle s vanilla benchmark model with recency effect \cite{Wang2016}\\
This is multiple linear regression model first proposed by \cite{Hong_2010}. It is used to produce benchmark scores for GEFCom2012. We use the extended version of this model \cite{Wang2016}, where the recency effect is added. We select the best average-lag pair ($d-h$ pair) that gives the best mean DTWE across 190 households in validation set.
\item Model B : Semi-parametric additive model \cite{semi_2014} \\
The semi-parametric additive model accommodates non-linear relationship between load and driver\textquotesingle s variables, We use the ST model (short term horizon) proposed by \cite{semi_2014} where non-linear functions of temperature, calendar variables and lagged demand are estimated using cubic spline regression. For each customer, we fit one model per instance of the day. The load is recorded every hour so that we fit 24 models corresponding to 24 hours per day.
\item Model C : Support Vector machine \cite{svm} \\
Support vector regression (SVR) applies risk minimization principle to minimize an upper bound of the generalization error\cite{Hong_2009}. This model was the winning entry in 2001 EUNITE competition in mid-term load forecasting whose goal is to predict daily maximum load (\cite{svm}). For comparison purpose, we apply SVR to forecast hourly load with the following input features, six binaries encoding Monday to Saturday respectively (Sunday is represented as all six attributes are set to zero), one numeric encoding temperature data, two numerics encoding past 24 hours load and past 48 hours load respectively. We select the model parameters by enumerating all 165 models based on different pairs of $\epsilon$ and cost, and choose the one that gives the best mean DTWE across households in validation set..
\item Model D : Artificial Neural Network \cite{ANNSTLF}\\
One of the implementation of ANN for short term load forecast is ANNSTLF \cite{ANNSTLF}, which includes base load ANN forecaster and load change forecaster. ANNSTLF based softwares were commercialized and are used by a large number of utilities across US and Canada \cite{hippert_ann}.
\end{enumerate}

 \subsection{Distance between PLD matrices}
The equivalence relation in equation \ref{DTW_equivalent} follows from 
\begin{align*}
\mbox{DTW}(A_x, A_y) & = \sum_{j} \mbox{dtw}(a_{xj}, a_{yj}) \\
                         & = \sum_{j} \mbox{dtw}(\frac{p_j x}{\alpha \|p\|_2^2}, \frac{p_j y}{\alpha \|p\|_2^2})\\
                         & = \sum_{j} \frac{p_j^2}{ \alpha^2 \|p\|_2^4} \mbox{dtw}(x, y)\\
                           & = \frac{\mbox{dtw}(x, y)}{\alpha^2 \|p\|_2^4}\sum_{j} p_j^2 \\
                          & = \frac{\mbox{dtw}(x, y)}{ \|\alpha p\|_2^2}.
\end{align*}

 \subsection{Proof for equation \eqref{theorem1}}
 \begin{lemma}
\label{lemma1}
\begin{align}
\mbox{tr}(A^TH) =  \mbox{tr}(\hat{A}^TH) = 0\label{claim1}
\end{align}
\end{lemma}
\begin{proof}
\begin{align*}
 \mbox{tr}(A^TH)         &\stackrel{(\grave{a})}= \frac{1}{\alpha \|p\|^2}\mbox{tr}(px^TH)
                                    = \frac{1}{\alpha \|p\|^2}\mbox{tr}(x^THp) 
                                    \stackrel{(\grave{b})} = 0 \\
 \mbox{tr}(\hat{A}^TH) &\stackrel{(\acute{a})}=  \frac{1}{\alpha \|p\|^2}\mbox{tr}(p\hat{x}^TH)
                                    = \frac{1}{\alpha \|p\|^2}\mbox{tr}(\hat{x}^THp) 
                                    \stackrel{(\acute{b})} = 0,
%\mbox{tr}(\hat{A}^TH) &\stackrel{(\acute{a})}=  \sum_{i} \langle p_i\hat{x}, h_i\rangle 
%                                    = \langle \hat{x}, \sum_{i}p_i h_i\rangle
%                                    = \langle \hat{x}, Hp\rangle 
%                                    \stackrel{(\acute{b})} = 0,
\end{align*}
where ($\grave{a}$) and ($\acute{a}$) follows from equation \ref{Asolution}, and ($\grave{b}$) and ($\acute{b}$) follows from $Hp = 0$ ($Ap = x$ and $\tilde{A}p = x$).
\end{proof}

 \begin{proof}
\begin{align*}
\|\hat{A}-\tilde{A}\|_F^2 &  = \| (\tilde{A} - A) - (\hat{A} - A)  \|_F^2 \\
                                     & = \|\hat{A} - A\|_F^2 + \|H\|_F^2 - 2 \mbox{tr}((\hat{A} - A)^TH) \\
                                     & \stackrel{(\acute{c})}{=} \frac{\|\hat{x} - x\|_2^2}{ \| \alpha p\|_2^2}  + \|H\|_F^2, \\
\| \tilde{A} \|_F^2  & = \| A + H\|_F^2 \\
                            &= \|A\|_F^2 + \|H\|_F^2 + 2\mbox{tr}(A^TH)\\
                            & \stackrel{(\acute{d})}{=}  \frac{\| x\|_2^2}{ \| \alpha p\|_2^2} + \|H\|_F^2
\end{align*}
where ($\acute{c}$) follows from equation \ref{l2_equivalent} and lemma \ref{lemma1}, and ($\acute{d}$) follows from lemma \ref{lemma1}. As a result, we have
%\begin{align*}
% \|\hat{A}-\tilde{A}\|_F^2 & \leq  \frac{\|\hat{x} - x\|_2^2}{\|p\|_2^2}  + R_H\sigma_{1}^2, \\
% \| \tilde{A} \|_F^2  &\leq   \frac{\| x\|_2^2}{\|p\|_2^2}  + R_H\sigma_{1}^2,
%\end{align*}

\begin{align*}
\frac{\|\hat{x} - x\|_2^2}{ \| \alpha p\|_2^2}  + \sigma_{max}^2 &\leq \|\hat{A}-\tilde{A}\|_F^2  \leq  \frac{\|\hat{x} - x\|_2^2}{ \| \alpha p\|_2^2}  + R_H\sigma_{max}^2, \\
 \frac{\| x\|_2^2}{ \| \alpha p\|_2^2}  + \sigma_{max}^2 & \leq \| \tilde{A} \|_F^2  \leq   \frac{\| x\|_2^2}{ \| \alpha p\|_2^2}  + R_H\sigma_{max}^2,
\end{align*}

 and therefore establish the proof.
 \end{proof}

\bibliography{template2}

\end{document}